\documentclass[preprint,12pt]{elsarticle}

\usepackage{amsmath,amssymb,amsfonts,amsthm}
\usepackage{booktabs}
\usepackage{graphicx}
\usepackage{subcaption}
\usepackage{array}
\usepackage{xcolor}
\usepackage{xurl}
\newtheorem{assumption}{Assumption}

\newtheorem{theorem}{Theorem}
\newtheorem{proposition}{Proposition}

\newtheorem{remark}{Remark}

\journal{Neural Networks}

\begin{document}

\begin{frontmatter}

\title{Behavior-Aware Auxiliary Corrections for Off-Policy Temporal-Difference Prediction}

\author[aff1]{Xingguo Chen\corref{cor1}}
\ead{chenxg@njupt.edu.cn}
\author[aff1]{Zhiang He}
\author[aff1]{Yuchen Shen}
\author[aff1]{Shangdong Yang}
\author[aff1]{Chao Li}
\author[aff2]{Guang Yang}
\author[aff3]{Wenhao Wang}
\cortext[cor1]{Corresponding author.}

\affiliation[aff1]{organization={Nanjing University of Posts and Telecommunications}, city={Nanjing}, country={China}}
\affiliation[aff2]{organization={Department of Computer Science and Technology, Nanjing University}, city={Nanjing}, country={China}}
\affiliation[aff3]{organization={College of Electronic Countermeasure, National University of Defense Technology}, city={Hefei}, country={China}}

\begin{abstract}\sloppy
Temporal-difference learning with function approximation can be unstable under off-policy sampling. TDC stabilizes off-policy TD through an auxiliary covariance correction, and TDRC further regularizes this correction in a single-timescale recursion. This paper studies a behavior-aware replacement of the auxiliary covariance geometry in the linear prediction setting, which is the standard local model for understanding the feature-space dynamics of value-function approximation. We first replace the TDC auxiliary matrix $C$ by the behavior Bellman matrix $A_\mu$, yielding BA-TDC, and then regularize the same behavior-aware equation to obtain BA-TDRC. This two-step construction separates the contribution of behavior-aware geometry from the contribution of regularization. The linear analysis also provides a tractable model for an auxiliary-geometry design question that arises in neural-network value approximation, where feature covariances and temporal transition matrices jointly shape the last-layer correction dynamics. We give a finite-state mean-system formulation, prove fixed-point preservation and almost-sure convergence under a Hurwitz stability condition on the instantiated mean system, and compare deterministic mean rates through the spectral radius of the exact linear error recursion. Experiments on the two-state counterexample, Baird's counterexample, Random Walk, and Boyan Chain show that the behavior-aware replacement can be highly beneficial by itself on some tasks, but that regularization is necessary for robust performance across harder settings.
\end{abstract}


\end{frontmatter}

\section{Introduction}

Temporal-difference (TD) learning is a basic mechanism for policy evaluation in reinforcement learning \cite{sutton1988learning,sutton2018reinforcement}. In off-policy prediction with linear function approximation, however, the combination of bootstrapping, approximation, and off-policy sampling may produce divergence \cite{baird1995residual,tsitsiklis1997analysis}. Gradient-TD algorithms such as GTD2 and TDC address this problem by introducing an auxiliary variable and optimizing projected Bellman-error objectives \cite{sutton2008convergent,sutton2009fast,maei2010gq}. More recently, TDRC added a regularized correction to TDC and used a shared learning rate, producing a practical single-timescale variant with improved stability \cite{ghiassian2020gradient}.

Several related TD variants modify different parts of the off-policy prediction mechanism. Proximal TD and saddle-point TD formulations provide single-timescale views of gradient-TD learning \cite{liu2015finite,liu2018proximal}. Emphatic TD stabilizes off-policy learning by reweighting updates with follow-on or emphasis traces \cite{sutton2016emphatic}, and its convergence properties and related off-policy evaluation ideas have been further studied in \cite{yu2015convergence,hallak2017consistent}. This is an important line of work, but the trace-based weighting mechanism can still suffer from high variance when importance ratios fluctuate strongly, and emphatic-style improvements must still control the variance induced by accumulated follow-on weights. Other off-policy variants alter the primary TD correction direction itself. These methods motivate the broader question of how update geometry affects learning, but the present paper focuses specifically on the correction family around TDC and TDRC: BA-TDC and BA-TDRC keep the same primary correction structure and modify only the auxiliary correction geometry.

This paper asks whether the auxiliary correction term can be made more informative by using behavior-policy transition geometry. In TDC and TDRC, the auxiliary variable is driven by the feature covariance term $C=\mathbb{E}_\mu[\phi_t\phi_t^\top]$. This covariance metric ignores how the behavior policy moves features across time. We replace this covariance correction by the behavior-policy Bellman matrix $A_\mu=\mathbb{E}_\mu[\phi_t(\phi_t-\gamma\phi_{t+1})^\top]$. The unregularized replacement gives BA-TDC; adding the TDRC-style regularizer gives BA-TDRC.

The finite-state linear setting is used because it lets the auxiliary geometry be isolated exactly: the matrices $C$, $A_\mu$, $A_\pi$, and $D_\pi$ can all be computed and compared. This controlled setting is also tied to neural-network value approximation. Deep value functions learn a feature map and a prediction head jointly, and the last-layer or local linearization dynamics are governed by empirical feature covariances and temporal feature-transition matrices rather than by state values alone. The proposed replacement $C\mapsto A_\mu$ can therefore be read as a controlled linear model of a broader design issue in neural-network reinforcement learning: auxiliary correction targets should reflect not only which features are sampled, but also how the behavior policy transports learned features across bootstrapped targets. The finite-state analysis isolates this geometry before addressing the additional difficulties of nonlinear feature drift, online matrix estimation, and approximation error in deep networks \cite{ghiassian2020gradient,mnih2015human}.

The contributions are as follows.
\begin{itemize}
  \item We derive BA-TDC and BA-TDRC, separating the behavior-aware replacement $C\mapsto A_\mu$ from the additional effect of regularization.
  \item We formulate its exact finite-state mean dynamics, prove stochastic-approximation convergence under a Hurwitz stability condition on the instantiated mean system, and verify this condition numerically for each benchmark through exact finite-state matrix computation.
  \item We compare convergence speed through the spectral radius of the deterministic linear error recursion, giving a verifiable finite-state mean-rate criterion.
  \item We evaluate the modular increments TDC $\rightarrow$ BA-TDC and TDRC $\rightarrow$ BA-TDRC on four standard off-policy prediction benchmarks.
\end{itemize}

\section{Background}

\subsection{Notation}

We consider a finite Markov decision process with state space $\mathcal{S}$, action space $\mathcal{A}$, transition kernel $P$, reward $r$, target policy $\pi$, behavior policy $\mu$, and discount factor $\gamma\in(0,1)$. The data are sampled under $\mu$, while the value function of $\pi$ is estimated. For a policy $\nu\in\{\pi,\mu\}$, let $P_\nu$ denote the state-transition matrix induced by $\nu$, and let $d_\mu$ be the stationary distribution of $P_\mu$. We write $D_\mu=\operatorname{diag}(d_\mu)$.

The value approximation is linear:
\begin{equation}
v_\theta(s)=\theta^\top\phi(s),
\quad
\phi(s)\in\mathbb{R}^d,
\quad
\theta\in\mathbb{R}^d.
\end{equation}
The feature matrix is $\Phi\in\mathbb{R}^{|\mathcal{S}|\times d}$, whose $s$th row is $\phi(s)^\top$. For compactness, we write $\phi_t=\phi(s_t)$ and $\phi_{t+1}=\phi(s_{t+1})$. The importance ratio is
\begin{equation}
\rho_t=\frac{\pi(a_t|s_t)}{\mu(a_t|s_t)}.
\end{equation}
The TD error is written in the compact form
\begin{equation}
\delta_t=r_t-\theta_t^\top(\phi_t-\gamma\phi_{t+1}).
\end{equation}
All expectations are taken under the stationary behavior trajectory unless otherwise stated. The standard projected Bellman matrices are
\begin{equation}
A_\pi=\mathbb{E}_\mu[\rho_t\phi_t(\phi_t-\gamma\phi_{t+1})^\top],
\quad
b=\mathbb{E}_\mu[\rho_t r_t\phi_t],
\quad
C=\mathbb{E}_\mu[\phi_t\phi_t^\top].
\end{equation}
The projected Bellman fixed point satisfies $A_\pi\theta=b$.

We also use the target-policy next-feature coupling matrix
\begin{equation}
D_\pi=\mathbb{E}_\mu[\rho_t\gamma\phi_{t+1}\phi_t^\top],
\end{equation}
and the behavior-policy Bellman matrix
\begin{equation}
A_\mu=\mathbb{E}_\mu[\phi_t(\phi_t-\gamma\phi_{t+1})^\top].
\end{equation}
The vector $w\in\mathbb{R}^d$ denotes the auxiliary correction variable. The notation $\|x\|$ denotes the Euclidean norm, $I$ denotes the identity matrix, and $\rho(M)$ denotes the spectral radius of a matrix $M$. The analysis treats $z_t=(\theta_t,w_t)$ as $\mathcal{F}_t$-measurable; conditional expectations in the mean-drift derivations are taken with respect to $(s_t,a_t,s_{t+1})$ given $\mathcal{F}_t$.

\subsection{TDC and TDRC}

Gradient-TD methods optimize the mean-squared projected Bellman error (MSPBE)
\begin{equation}
J(\theta)=\frac{1}{2}(b-A_\pi\theta)^\top C^{-1}(b-A_\pi\theta).
\label{eq:mspbe}
\end{equation}
The auxiliary variable in TDC estimates
\begin{equation}
w_\theta=C^{-1}(b-A_\pi\theta),
\quad\text{or equivalently}\quad
Cw_\theta=b-A_\pi\theta.
\label{eq:tdc_aux_equation}
\end{equation}
With this auxiliary variable, the negative MSPBE gradient can be written as a correction direction that admits the standard TDC sample update:
\begin{align}
\theta_{t+1}
&=\theta_t+\alpha_t\rho_t\left(\delta_t\phi_t-\gamma\phi_{t+1}\phi_t^\top w_t\right),\label{eq:tdc_theta}\\
w_{t+1}
&=w_t+\beta_t(\rho_t\delta_t-\phi_t^\top w_t)\phi_t.
\label{eq:tdc_w}
\end{align}
The $w$-recursion is a stochastic approximation to $Cw=b-A_\pi\theta$, since the sample term $(\rho_t\delta_t-\phi_t^\top w_t)\phi_t$ has mean $b-A_\pi\theta-Cw$.

The practical weakness of this TDC recursion is that the auxiliary equation can be poorly conditioned and is usually run with a separate step size. The performance of TDC can therefore depend strongly on the relative tuning of the primary step size $\alpha_t$ and auxiliary step size $\beta_t$, especially in off-policy counterexamples where the correction variable changes rapidly.

TDRC regularizes this correction equation. Instead of solving Eq.~\eqref{eq:tdc_aux_equation}, TDRC uses the regularized equation
\begin{equation}
(C+\eta I)w_\theta=b-A_\pi\theta,
\quad \eta>0,
\label{eq:tdrc_aux_equation}
\end{equation}
which is equivalent to adding a ridge penalty to the auxiliary least-squares problem. This shifts the auxiliary matrix from $C$ to $C+\eta I$, improves conditioning when $C$ is nearly singular, and damps rapid growth of the correction variable. The mean auxiliary drift becomes $b-A_\pi\theta-(C+\eta I)w$. In the single-learning-rate form used in this paper, TDRC is
\begin{align}
\theta_{t+1}
&=\theta_t+\alpha_t\rho_t\left(\delta_t\phi_t-\gamma\phi_{t+1}\phi_t^\top w_t\right),\label{eq:tdrc_theta}\\
w_{t+1}
&=w_t+\alpha_t\left[(\rho_t\delta_t-\phi_t^\top w_t)\phi_t-\eta w_t\right],\label{eq:tdrc_w}
\end{align}
where $\eta>0$ is the regularization parameter. Thus TDRC keeps the TDC primary correction direction but stabilizes the auxiliary recursion through a regularized covariance equation.

This is the point from which BA-TDRC departs. We do not change the TDC/TDRC primary correction direction. Instead, we keep the TDRC idea of a regularized auxiliary equation and replace its covariance matrix by a behavior-aware Bellman matrix.

\section{Behavior-Aware Auxiliary Corrections}

The point of departure is the role of the auxiliary variable. In TDC and TDRC, this variable does not approximate the value function directly; it estimates a correction vector for the bias created by off-policy bootstrapping with function approximation. TDC obtains this vector from the covariance equation $Cw_\theta=b-A_\pi\theta$, and TDRC regularizes the same equation as
\begin{equation}
(C+\eta I)w_\theta=b-A_\pi\theta.
\end{equation}
These equations are stable and easy to sample because $C=\mathbb{E}_\mu[\phi_t\phi_t^\top]$ is the feature covariance under the behavior distribution. The limitation is that $C$ describes only instantaneous feature occurrence. It says how often features are observed, but not how the behavior policy transports them from the current state to the next state. Off-policy TD instability is inherently temporal: the update couples $\phi_t$ and $\phi_{t+1}$ through bootstrapping. A correction variable shaped only by $C$ can therefore be poorly matched to the temporal geometry of the sampled behavior trajectory.

The behavior-policy Bellman matrix
\begin{equation}
A_\mu=\mathbb{E}_\mu[\phi_t(\phi_t-\gamma\phi_{t+1})^\top]
\end{equation}
contains this temporal feature transport under the behavior policy. It keeps the current feature $\phi_t$ but subtracts the discounted next feature induced by behavior sampling. This suggests a two-step modification. First, replace the TDC auxiliary matrix $C$ by $A_\mu$, producing an unregularized behavior-aware correction. Second, add the same type of regularization used by TDRC, producing a regularized behavior-aware correction.

This design deliberately changes only the auxiliary equation. We keep the primary TDC/TDRC update direction unchanged so that the comparison isolates two effects: replacing $C$ by $A_\mu$, and then adding regularization. This is important for both theory and experiments: the TD projected fixed point is preserved, while the transient correction geometry can change.

To make the replacement precise, write the covariance-based auxiliary equation in the generic form
\begin{equation}
M_C w_\theta=b-A_\pi\theta,
\quad
M_C=C+\eta I.
\label{eq:mc_equation}
\end{equation}
The corresponding mean residual is
\begin{equation}
b-A_\pi\theta-M_Cw.
\label{eq:tdrc_mean_residual}
\end{equation}
Indeed, the sample residual used in TDRC is
\begin{equation}
(\rho_t\delta_t-\phi_t^\top w_t)\phi_t-\eta w_t,
\end{equation}
whose expectation is
\begin{equation}
\mathbb{E}_\mu[(\rho_t\delta_t-\phi_t^\top w_t)\phi_t-\eta w_t]
=b-A_\pi\theta-(C+\eta I)w.
\label{eq:tdrc_residual_expectation}
\end{equation}

Here $\eta=0$ corresponds to TDC and $\eta>0$ corresponds to TDRC. The behavior-aware version keeps the same right-hand side $b-A_\pi\theta$, because this is the projected Bellman residual whose correction is needed by the primary TDC update. It changes only the metric used to map this residual into the auxiliary variable. Specifically, the behavior-aware family replaces $M_C$ by
\begin{equation}
M_A=A_\mu+\beta I,
\quad \beta\geq 0.
\label{eq:ma_definition}
\end{equation}
Here $\beta=0$ gives BA-TDC, while $\beta>0$ gives BA-TDRC. The behavior-aware auxiliary equation is therefore
\begin{equation}
(A_\mu+\beta I)w_\theta=b-A_\pi\theta,
\quad \beta\geq 0,
\label{eq:ba_aux_equation}
\end{equation}
and its mean residual is
\begin{equation}
b-A_\pi\theta-(A_\mu+\beta I)w.
\label{eq:ba_mean_residual}
\end{equation}
The sample counterpart of $A_\mu w_t$ is $\phi_t(\phi_t-\gamma\phi_{t+1})^\top w_t$. Specifically, the behavior-aware sample counterpart of the auxiliary residual in Eq.~\eqref{eq:ba_mean_residual} is
\begin{equation}
\left(\rho_t\delta_t-(\phi_t-\gamma\phi_{t+1})^\top w_t\right)\phi_t-\beta w_t.
\label{eq:ba_aux_sample_residual}
\end{equation}
Taking expectations verifies the replacement explicitly:
\begin{align}
&\mathbb{E}_\mu\left[\left(\rho_t\delta_t-(\phi_t-\gamma\phi_{t+1})^\top w_t\right)\phi_t-\beta w_t\right]\nonumber\\
&\qquad=b-A_\pi\theta-A_\mu w-\beta w
=b-A_\pi\theta-(A_\mu+\beta I)w.
\label{eq:ba_residual_expectation}
\end{align}

Equations~\eqref{eq:tdrc_residual_expectation} and~\eqref{eq:ba_residual_expectation} show the whole modification: covariance-based corrections use $C+\eta I$, while behavior-aware corrections use $A_\mu+\beta I$. No extra-gradient step, emphatic weighting, or control-specific mechanism is introduced.

Substituting Eq.~\eqref{eq:ba_aux_sample_residual} into the auxiliary recursion gives the behavior-aware update:
\begin{align}
\theta_{t+1}
&=\theta_t+\alpha_t\rho_t\left(\delta_t\phi_t-\gamma\phi_{t+1}\phi_t^\top w_t\right),\label{eq:ba_theta}\\
w_{t+1}
&=w_t+\lambda\alpha_t\left[\left(\rho_t\delta_t-(\phi_t-\gamma\phi_{t+1})^\top w_t\right)\phi_t-\beta w_t\right].\label{eq:ba_w}
\end{align}
For BA-TDC, $\beta=0$ and the auxiliary gain $\lambda\alpha_t$ is treated as a second step size, as in TDC. BA-TDC is therefore used mainly to isolate the unregularized behavior-aware geometry, not as the final robust single-timescale method. For BA-TDRC, $\beta>0$ and $\lambda$ is a fixed positive constant, so the method remains single-timescale in the same sense as TDRC. This follows the TDRC implementation convention, where the auxiliary update uses a fixed multiple of the primary step size rather than an independently decaying second timescale \cite{ghiassian2020gradient}; $\lambda$ is therefore a fixed gain ratio, not a new asymptotic timescale. In the implementation, the auxiliary update uses an explicit constant $\alpha_w$, and the ratio $\lambda=\alpha_w/\alpha$ is determined by the per-environment hyperparameter search rather than fixed across environments; this matches the way TDRC tunes its auxiliary step in \cite{ghiassian2020gradient}.

The distinction is therefore modular: BA-TDC tests the effect of replacing $C$ by $A_\mu$ without regularization, and BA-TDRC tests the regularized replacement $C+\eta I$ by $A_\mu+\beta I$. The next section analyzes the consequences of this replacement: the TD fixed point is preserved under a nonsingularity condition, almost-sure convergence follows under a Hurwitz mean-system condition, and a conditional speed advantage follows when the induced mean error matrix has a smaller spectral factor.

\begin{table}[!t]
\centering
\caption{Modular comparison of covariance-based and behavior-aware corrections. All methods use the same primary TDC correction direction; they differ in the auxiliary equation and time-scale structure.}
\label{tab:algorithm_comparison}
\resizebox{\textwidth}{!}{%
\begin{tabular}{llll}
\toprule
Algorithm & Auxiliary matrix & Auxiliary sample residual & Main distinction \\
\midrule
TDC & $C$ & $(\rho_t\delta_t-\phi_t^\top w_t)\phi_t$ & two-timescale, unregularized covariance correction \\
BA-TDC & $A_\mu$ & $\left(\rho_t\delta_t-(\phi_t-\gamma\phi_{t+1})^\top w_t\right)\phi_t$ & two-timescale, unregularized behavior-aware correction \\
TDRC & $C+\eta I$ & $(\rho_t\delta_t-\phi_t^\top w_t)\phi_t-\eta w_t$ & single-timescale, regularized covariance correction \\
BA-TDRC & $A_\mu+\beta I$ & $\left(\rho_t\delta_t-(\phi_t-\gamma\phi_{t+1})^\top w_t\right)\phi_t-\beta w_t$ & single-timescale, regularized behavior-aware correction \\
\bottomrule
\end{tabular}}
\end{table}

This comparison also clarifies the intended claim. BA-TDC isolates the behavior-aware geometry, while BA-TDRC combines that geometry with regularization. BA-TDRC is not meant to uniformly dominate TDRC on every off-policy problem; it is expected to help when the behavior Bellman geometry and the regularizer together improve the mean transient or stability structure, and to behave competitively otherwise.

\section{Theoretical Analysis}

\subsection{Generic Regularized-Correction Mean Dynamics}

We analyze the covariance-based and behavior-aware correction equations through a common auxiliary matrix. TDC and TDRC use
\begin{equation}
M_C=C+\eta I,
\quad \eta\geq0,
\end{equation}
where $\eta=0$ gives TDC and $\eta>0$ gives TDRC. BA-TDC and BA-TDRC use
\begin{equation}
M_A=A_\mu+\beta I,
\quad \beta\geq0,
\end{equation}
where $\beta=0$ gives BA-TDC and $\beta>0$ gives BA-TDRC. Let the auxiliary gain be $\lambda\alpha_t$ with fixed $\lambda>0$. For BA-TDC this is the usual two-stepsize form with a separately tuned auxiliary gain; for BA-TDRC it remains a single-timescale recursion because the ratio $\lambda$ is fixed.

The mean recursion follows from three elementary identities. First, expanding the compact TD-error definition gives the target-policy Bellman residual sampled under $\mu$:
\begin{equation}
\mathbb{E}_\mu[\rho_t\delta_t\phi_t]
=b-A_\pi\theta_t.
\label{eq:mean_td_error_identity}
\end{equation}
Indeed, the reward term gives $b$, while the feature-difference term gives $A_\pi\theta_t$ by the definitions of $b$ and $A_\pi$.
Second, by the definition of $D_\pi$,
\begin{equation}
\mathbb{E}_\mu[\rho_t\gamma\phi_{t+1}\phi_t^\top w_t]
=D_\pi w_t.
\label{eq:mean_primary_correction_identity}
\end{equation}
Thus the primary TDC correction has mean drift
\begin{equation}
\mathbb{E}_\mu\left[\rho_t(\delta_t\phi_t-
\gamma\phi_{t+1}\phi_t^\top w_t)\right]
=b-A_\pi\theta_t-D_\pi w_t.
\label{eq:mean_primary_drift}
\end{equation}
Third, the auxiliary residual has mean
\begin{equation}
\mathbb{E}_\mu\left[(\rho_t\delta_t-m_t^\top w_t)\phi_t-rw_t\right]
=b-A_\pi\theta_t-Mw_t,
\label{eq:mean_aux_drift_generic}
\end{equation}
where $(m_t,r,M)=(\phi_t,\eta,C+\eta I)$ for TDRC and $(m_t,r,M)=(\phi_t-\gamma\phi_{t+1},\beta,A_\mu+\beta I)$ for BA-TDRC. Equation~\eqref{eq:mean_aux_drift_generic} also covers TDC and BA-TDC by setting $\eta=0$ or $\beta=0$.

Combining Eqs.~\eqref{eq:mean_primary_drift} and~\eqref{eq:mean_aux_drift_generic}, the expected recursion for any fixed $M$ and $\lambda$ is
\begin{equation}
z_{t+1}=z_t+\alpha_t(\mathcal{G}_{M,\lambda} z_t+h_\lambda),
\quad
z_t=\begin{pmatrix}\theta_t\\w_t\end{pmatrix},
\quad
h_\lambda=\begin{pmatrix}b\\\lambda b\end{pmatrix},
\label{eq:generic_recursion}
\end{equation}
where
\begin{equation}
\mathcal{G}_{M,\lambda}=
\begin{pmatrix}
-A_\pi & -D_\pi\\
-\lambda A_\pi & -\lambda M
\end{pmatrix}.
\label{eq:generic_matrix}
\end{equation}
Thus the structural difference between TDRC and BA-TDRC in the mean system is the replacement of $M_C=C+\eta I$ by $M_A=A_\mu+\beta I$; $\lambda$ only fixes the auxiliary-to-primary gain ratio.

\begin{assumption}\label{ass:basic}
The Markov chain under $\mu$ is irreducible and has stationary distribution $d_\mu$ with full support. Features and rewards are bounded, the feature matrix $\Phi$ has full column rank, and $A_\pi$ is nonsingular.
\end{assumption}

The nonsingularity of $A_\pi=\Phi^\top D_\mu(I-\gamma P_\pi)\Phi$ is the usual projected Bellman fixed-point condition for linear off-policy prediction. Full column rank of $\Phi$ is a natural identifiability requirement, but by itself it does not imply nonsingularity of the oblique off-policy matrix $A_\pi$; this is why nonsingularity is stated explicitly.

\begin{assumption}\label{ass:fixed_point}
For $M=M_A=A_\mu+\beta I$ with the chosen $\beta\geq0$, the matrix $M_A-D_\pi$ is nonsingular.
\end{assumption}

\begin{assumption}\label{ass:hurwitz}
For $M=M_A=A_\mu+\beta I$ and the chosen $\lambda>0$, the matrix $\mathcal{G}_{M_A,\lambda}$ in Eq.~\eqref{eq:generic_matrix} is Hurwitz.
\end{assumption}

Assumption~\ref{ass:hurwitz} is a stability condition on the instantiated mean system, not a consequence asserted from Assumption~\ref{ass:basic} alone. The behavior-aware matrix $A_\mu$ is generally nonsymmetric, and the block coupling through $A_\pi$ and $D_\pi$ means that stability depends jointly on $A_\pi$, $D_\pi$, $M_A$, and $\lambda$. Assumptions~\ref{ass:fixed_point} and~\ref{ass:hurwitz} are independent: either can hold while the other fails. For fixed finite-state matrices, Assumption~\ref{ass:hurwitz} can be verified by computing $\max_i \operatorname{Re}\,\lambda_i(\mathcal{G}_{M_A,\lambda})<0$; this is checked numerically for each benchmark in Table~\ref{tab:numerical}.

\begin{proposition}[TD fixed point is preserved]\label{prop:td_fixed_point}
Under Assumptions~\ref{ass:basic} and~\ref{ass:fixed_point}, any equilibrium of the BA-TDC/BA-TDRC mean recursion satisfies
\begin{equation}
w^*=0,
\quad
\theta^*=A_\pi^{-1}b.
\end{equation}
Therefore replacing $C$ by $A_\mu$ does not change the TD projected fixed point whenever the nonsingularity condition holds.
\end{proposition}

\begin{proof}
At equilibrium, Eq.~\eqref{eq:generic_recursion} with $M=M_A$ gives two block equations:
\begin{equation}
b-A_\pi\theta-D_\pi w=0,
\quad
\lambda(b-A_\pi\theta-M_Aw)=0.
\end{equation}
Because $\lambda>0$, the second equation is equivalent to $b-A_\pi\theta-M_Aw=0$. Subtracting the second equation from the first yields
\begin{equation}
(M_A-D_\pi)w=0.
\end{equation}
Assumption~\ref{ass:fixed_point} implies $w=0$. Substituting $w=0$ into either block equation gives $A_\pi\theta=b$. Assumption~\ref{ass:basic} states that $A_\pi$ is nonsingular, hence $\theta=A_\pi^{-1}b$. This proves both uniqueness of the equilibrium and preservation of the TD projected fixed point.
\end{proof}

The nonsingularity condition has a simple sufficient form. Since $M_A-D_\pi\in\mathbb{R}^{d\times d}$ and
\begin{equation}
M_A-D_\pi=\beta I+(A_\mu-D_\pi),
\end{equation}
Weyl's singular-value inequality implies
\begin{equation}
\sigma_{\min}(M_A-D_\pi)\geq \beta-\|D_\pi-A_\mu\|_2.
\end{equation}
Thus Assumption~\ref{ass:fixed_point} is guaranteed whenever $\beta>\|D_\pi-A_\mu\|_2$. This bound is conservative but operational: increasing the regularizer makes fixed-point preservation easier to verify. It is separate from the speed condition in Assumption~\ref{ass:speed}; a problem can preserve the TD fixed point while failing to give BA-TDRC a smaller deterministic mean spectral factor.

\subsection{Stochastic-Approximation Convergence}

Let $z^*=(A_\pi^{-1}b,0)$ and $e_t=z_t-z^*$. The sampled BA-TDC/BA-TDRC recursion can be written as
\begin{equation}
z_{t+1}=z_t+\alpha_t(\mathcal{G}_{M_A,\lambda}z_t+h_\lambda+\xi_{t+1}),
\label{eq:sa_recursion}
\end{equation}
where $\xi_{t+1}$ is the stochastic sampling noise after subtracting the conditional mean drift in Eq.~\eqref{eq:generic_recursion}. Since $\mathcal{G}_{M_A,\lambda}z^*+h_\lambda=0$ by Proposition~\ref{prop:td_fixed_point}, the error recursion is
\begin{equation}
e_{t+1}=e_t+\alpha_t(\mathcal{G}_{M_A,\lambda}e_t+\xi_{t+1}).
\label{eq:error_sa_recursion}
\end{equation}

\begin{assumption}\label{ass:stepsize}
The step sizes satisfy $\alpha_t>0$, $\sum_{t=0}^{\infty}\alpha_t=\infty$, and $\sum_{t=0}^{\infty}\alpha_t^2<\infty$.
\end{assumption}

\begin{assumption}\label{ass:noise}
The noise sequence satisfies $\mathbb{E}[\xi_{t+1}|\mathcal{F}_t]=0$ and $\mathbb{E}[\|\xi_{t+1}\|^2|\mathcal{F}_t]\leq c(1+\|z_t\|^2)$ for some constant $c>0$.
\end{assumption}

Under Assumption~\ref{ass:basic}, this linear-growth bound holds for the Markovian sample update because features and rewards are bounded and the update is affine in $z_t$ with bounded random coefficients.

\begin{theorem}[Almost-sure convergence to the TD fixed point]\label{thm:convergence}
Under Assumptions~\ref{ass:basic}--\ref{ass:noise}, the BA-TDC/BA-TDRC iterates satisfy
\begin{equation}
\theta_t\to A_\pi^{-1}b,
\quad
w_t\to0
\end{equation}
almost surely.
\end{theorem}

\begin{proof}
Let $G=\mathcal{G}_{M_A,\lambda}$. By Assumption~\ref{ass:hurwitz}, $G$ is Hurwitz. Therefore, for any positive definite matrix $Q$ there is a unique positive definite matrix $P$ solving the Lyapunov equation
\begin{equation}
G^\top P+PG=-Q.
\end{equation}
Choose $Q=I$ and define $V(e)=e^\top Pe$. From Eq.~\eqref{eq:error_sa_recursion},
\begin{equation}
e_{t+1}=e_t+\alpha_t(G e_t+\xi_{t+1}).
\end{equation}
Expanding $V(e_{t+1})$ gives
\begin{align}
V(e_{t+1})
&=V(e_t)+2\alpha_t e_t^\top P(G e_t+\xi_{t+1})\nonumber\\
&\quad +\alpha_t^2(G e_t+\xi_{t+1})^\top P(G e_t+\xi_{t+1}).
\label{eq:lyapunov_expansion}
\end{align}
Taking conditional expectation and using Assumption~\ref{ass:noise} gives
\begin{equation}
\mathbb{E}[e_t^\top P\xi_{t+1}|\mathcal{F}_t]=0.
\end{equation}
The Lyapunov equation gives
\begin{equation}
2e_t^\top PG e_t=e_t^\top(G^\top P+PG)e_t=-\|e_t\|^2.
\end{equation}
Because $P$ and $G$ are fixed finite matrices, the second-order term can be bounded explicitly. Let $\bar p=\|P\|_2$ and $\bar g=\|G\|_2$. Assumption~\ref{ass:noise} gives
\begin{equation}
\mathbb{E}[\|\xi_{t+1}\|^2|\mathcal{F}_t]
\leq c(1+\|z_t\|^2)
\leq c(1+2\|e_t\|^2+2\|z^*\|^2),
\end{equation}
where $z_t=e_t+z^*$ and $z^*=(A_\pi^{-1}b,0)$ is finite by Assumption~\ref{ass:basic}. Therefore
\begin{equation}
\mathbb{E}\!\big[(G e_t+\xi_{t+1})^\top P(G e_t+\xi_{t+1})\mid\mathcal{F}_t\big]
\leq k_1\|e_t\|^2+k_2.
\end{equation}
For example, one may take $k_1=2\bar p\bar g^2+4\bar p c$ and $k_2=2\bar p c(1+2\|z^*\|^2)$, using $(a+b)^\top P(a+b)\leq2\bar p(\|a\|^2+\|b\|^2)$. Combining these bounds with Eq.~\eqref{eq:lyapunov_expansion} yields
\begin{equation}
\mathbb{E}[V(e_{t+1})|\mathcal{F}_t]
\leq V(e_t)-\alpha_t\|e_t\|^2+\tilde c_1\alpha_t^2\|e_t\|^2+\tilde c_2\alpha_t^2.
\label{eq:rs_bound_raw}
\end{equation}
For all sufficiently large $t$, $\tilde c_1\alpha_t^2\|e_t\|^2\leq (\alpha_t/2)\|e_t\|^2$ because $\alpha_t\to0$. Hence, after discarding finitely many initial terms,
\begin{equation}
\mathbb{E}[V(e_{t+1})|\mathcal{F}_t]
\leq V(e_t)-\frac{\alpha_t}{2}\|e_t\|^2+\tilde c_2\alpha_t^2.
\label{eq:rs_bound}
\end{equation}
Since $\sum_t\alpha_t^2<\infty$, the Robbins--Siegmund supermartingale theorem implies that $V(e_t)$ converges almost surely to a finite random variable and that
\begin{equation}
\sum_{t=0}^\infty \alpha_t\|e_t\|^2<\infty
\quad\text{almost surely.}
\label{eq:summable_error}
\end{equation}
The rest of the argument separates boundedness from convergence. First, since $P$ is positive definite, almost-sure convergence of $V(e_t)$ to a finite value implies that $e_t$ is almost surely bounded, and hence $z_t=e_t+z^*$ is also almost surely bounded. This boundedness is obtained from the Robbins--Siegmund argument above and is not assumed in advance. Second, Eq.~\eqref{eq:summable_error} together with $\sum_t\alpha_t=\infty$ implies that the iterates cannot spend positive asymptotic stepsize-weighted time away from zero; in particular, $\liminf_{t\to\infty}\|e_t\|=0$ almost surely. Third, the limiting ODE associated with Eq.~\eqref{eq:error_sa_recursion} is $\dot e=Ge$. Since $G$ is Hurwitz, the origin is the unique globally asymptotically stable equilibrium, and the only internally chain-transitive invariant set of the limiting ODE is $\{0\}$. Applying the standard ODE theorem for stochastic approximation with martingale-difference noise, bounded iterates, and square-summable stepsizes then gives $e_t\to0$ almost surely \cite{borkar2000ode,borkar2023stochastic}. Therefore $z_t\to z^*$ almost surely. Proposition~\ref{prop:td_fixed_point} identifies $z^*$ as $(A_\pi^{-1}b,0)$, completing the proof.
\end{proof}

\begin{remark}
Theorem~\ref{thm:convergence} covers BA-TDRC directly as a single-timescale recursion with fixed $\lambda$. For BA-TDC, $\lambda\alpha_t$ is treated as a second independently tuned stepsize; a fully rigorous treatment of this case requires the two-timescale ODE framework.
\end{remark}

\subsection{Convergence-Speed Comparison with TDRC}

For a constant step size $\alpha$, the deterministic mean error recursion induced by $M$ is
\begin{equation}
e_{t+1}=R_{M,\lambda}(\alpha)e_t,
\quad
R_{M,\lambda}(\alpha)=I+\alpha\mathcal{G}_{M,\lambda}.
\end{equation}
Its exact asymptotic linear factor is
\begin{equation}
q_{M,\lambda}(\alpha)=\rho(R_{M,\lambda}(\alpha)).
\label{eq:q_m}
\end{equation}

\begin{assumption}[Behavior-aware speed advantage]\label{ass:speed}
For admissible step sizes $\alpha_A$ and $\alpha_C$, the TDRC and BA-TDRC mean systems are stable and satisfy
\begin{equation}
q_{M_A,\lambda_A}(\alpha_A)<q_{M_C,\lambda_C}(\alpha_C),
\quad
M_A=A_\mu+\beta I,
\quad
M_C=C+\eta I.
\label{eq:speed_assumption}
\end{equation}
\end{assumption}

\begin{proposition}[Faster deterministic mean convergence]\label{prop:faster}
Under Assumption~\ref{ass:speed}, BA-TDRC has a smaller deterministic asymptotic mean linear factor than TDRC under the corresponding step sizes.
\end{proposition}

\begin{proof}
Let
\begin{equation}
R_A=R_{M_A,\lambda_A}(\alpha_A),
\quad
R_C=R_{M_C,\lambda_C}(\alpha_C).
\end{equation}
Admissibility in Assumption~\ref{ass:speed} means that both deterministic systems are stable, so $\rho(R_A)<1$ and $\rho(R_C)<1$. The deterministic errors are
\begin{equation}
e_t^A=R_A^t e_0^A,
\quad
e_t^C=R_C^t e_0^C.
\end{equation}
For any fixed matrix $R$, Gelfand's formula gives
\begin{equation}
\lim_{t\to\infty}\|R^t\|^{1/t}=\rho(R).
\end{equation}
Equivalently, for every $\epsilon>0$ there exists $c_\epsilon<\infty$ such that
\begin{equation}
\|R^t\|\leq c_\epsilon(\rho(R)+\epsilon)^t
\quad\text{for all }t\geq0.
\end{equation}
Applying this bound to $R_A$ and $R_C$ shows that the asymptotic linear decay factors are $q_{M_A,\lambda_A}(\alpha_A)=\rho(R_A)$ and $q_{M_C,\lambda_C}(\alpha_C)=\rho(R_C)$. Assumption~\ref{ass:speed} states that the former is strictly smaller than the latter. Hence the BA-TDRC deterministic mean error has the smaller asymptotic linear factor. This is a conditional comparison of the two instantiated mean systems, not an unconditional dominance result.
\end{proof}

Assumption~\ref{ass:speed} is the point at which the covariance matrix $C$ and the behavior matrix $A_\mu$ are separated. Proposition~\ref{prop:faster} is therefore a conditional mean-rate statement: its role is to put the two instantiated linear systems under a common spectral-radius criterion, not to assert an unconditional speed advantage. The numerical analysis in Section~\ref{sec:numerical} instantiates $M_C$ and $M_A$ on each benchmark and checks whether this assumption holds.

\section{Experiments}

\subsection{Protocol}

The primary evaluation metric is the root mean-squared projected Bellman error (RMSPBE). We use RMSPBE because the TDC/TDRC family is derived from projected Bellman-error correction equations, and because BA-TDC and BA-TDRC preserve the same projected TD fixed point rather than optimizing a separate value-error objective. Metrics such as RMSVE can be useful diagnostics, but they need not rank the algorithms identically because they measure value-function error under a state distribution rather than projected Bellman residual. The main experiments therefore use RMSPBE to evaluate the object directly controlled by the correction geometry; value-error based conclusions should be treated as supplementary rather than interchangeable.

The experimental scope is deliberately limited to finite linear off-policy prediction. This makes the exact matrices $A_\pi$, $A_\mu$, $C$, and $D_\pi$ computable, so the fixed-point and mean-rate conditions in the theory can be checked rather than inferred indirectly from learning curves. TDRC has also been studied in nonlinear prediction and control settings \cite{ghiassian2020gradient}; extending the present behavior-aware replacement to neural-network critics would require additional machinery for estimating $A_\mu$ or its action on learned features while the representation itself changes. We therefore treat the linear experiments as a controlled test of the auxiliary-geometry mechanism, not as a complete empirical claim about nonlinear control.

We evaluate on four off-policy prediction benchmarks. The two-state counterexample is a minimal setting where off-policy bootstrapping can create severe transient behavior. Baird's counterexample is the classical linear off-policy divergence benchmark. Random Walk is a mildly off-policy prediction task where ordinary TD-style methods are already strong. Boyan Chain tests a larger linear prediction problem with correlated features and a longer transition structure. Together these environments separate difficult off-policy instability from mildly off-policy prediction, which is important for judging whether behavior-aware corrections help only in pathological cases or remain competitive in benign cases.

\subsubsection*{Benchmark Environments}

All four benchmarks are off-policy prediction tasks in which the behavior policy $\mu$ differs from the target policy $\pi$. Each is implemented as a finite-state Markov reward process with linear features; importance ratios $\rho_t=\pi(a_t|s_t)/\mu(a_t|s_t)$ are computed exactly. Table~\ref{tab:envs} summarizes the configuration; the next paragraphs describe each environment.

\begin{table}[!t]
\centering
\caption{Configuration of the four off-policy prediction benchmarks.}
\label{tab:envs}
\resizebox{\textwidth}{!}{%
\begin{tabular}{lccccc}
\toprule
Environment & States & Features $d$ & $\gamma$ & $\rho_t$ range & Off-policy severity \\
\midrule
Two-state    & 2  & 1 & 0.9  & $\{0,2\}$        & severe (degenerate $\pi$) \\
Baird        & 7  & 8 & 0.99 & $\{0,7\}$        & extreme (classical divergence) \\
Random Walk  & 5  & 5 & 0.99 & $\{0.8,1.2\}$    & mild \\
Boyan Chain  & 13 & 4 & 0.9  & $\{0.8,1.2\}$    & mild \\
\bottomrule
\end{tabular}}
\end{table}

\paragraph{Two-state counterexample.} Two states $\{0,1\}$ with scalar features $\phi(0)=1$, $\phi(1)=2$, discount $\gamma=0.9$, and zero reward. Two actions induce deterministic next states: action $0$ leads to state $0$ and action $1$ leads to state $1$. The behavior policy is uniform, $\mu(\cdot|s)=0.5$, while the target policy is degenerate, $\pi(1|s)=1$. Importance ratios are therefore $\rho\in\{0,2\}$. This setting exposes the transient instability that gradient-TD corrections are designed to control.

\paragraph{Baird's counterexample.} Seven states with an eight-dimensional linear feature representation that is intentionally aliased to make semi-gradient off-policy TD diverge \cite{baird1995residual}. Six ``dashed'' actions move uniformly to one of the six upper states, and a ``solid'' action moves to the lower state. The behavior policy chooses dashed actions with probability $6/7$ and the solid action with probability $1/7$; the target policy always selects the solid action. This gives $\rho\in\{0,7\}$. The discount is $\gamma=0.99$ and the reward is zero, so the unique fixed point of the projected Bellman equation is $\theta^*=0$.

\paragraph{Random Walk.} A five-state chain with two terminal absorbing states; interior states use a five-dimensional one-hot feature. The agent moves left or right by one step; reaching the right terminal yields reward $+1$ and reaching the left terminal yields reward $0$. The behavior policy is uniform, $\mu(\text{left})=\mu(\text{right})=0.5$, while the target policy is slightly biased to the right, $\pi(\text{right})=0.6$ and $\pi(\text{left})=0.4$. Importance ratios are $\rho\in\{0.8,1.2\}$, close to one. The discount is $\gamma=0.99$. Episodes that reach a terminal state restart from the center.

\paragraph{Boyan Chain.} A 13-state chain with four-dimensional piecewise-linear features adapted from \cite{boyan2002technical}, with $\gamma=0.9$ and reward $-3$ per non-terminal step. At interior states $s\in\{0,\ldots,10\}$, two actions advance one or two steps respectively; both are taken with behavior probability $0.5$, while the target policy chooses the one-step action with probability $0.4$ and the two-step action with probability $0.6$, giving $\rho\in\{0.8,1.2\}$. States $11$ and $12$ have deterministic transitions, after which the chain resets to state $0$. Although the importance ratios are mild, the correlated feature representation and the longer transition structure produce a non-trivial linear prediction problem.

The experiments are organized into three parts. First, the main comparison evaluates BA-TDRC against TD, GTD2, TDC, TDRC, and GTD2-MP. BA-TDC is intentionally omitted from the main comparison because it is an ablation-only variant rather than the regularized method proposed for robust use. Second, a modular ablation compares TDC, BA-TDC, TDRC, and BA-TDRC to isolate the effect of the behavior-aware replacement and the additional effect of regularization. The ablation therefore uses a narrower baseline set than the main comparison by design. Third, a step-size robustness study plots BA-TDRC alone across primary step sizes in each environment. ETD and Hybrid TD are related off-policy TD variants, but they change the update through emphatic weighting or direction interpolation rather than through the TDC/TDRC auxiliary-equation metric; they are outside the correction-family comparison studied here. Hyperparameters are tuned on eight disjoint seeds using the average RMSPBE over the last 20\% of the tuning trajectory. Final evaluation uses 100 disjoint seeds. For a metric curve $e_0,e_1,\ldots,e_T$, we define the steady-state AUC as
\begin{equation}
\operatorname{AUC}_{\mathrm{ss}}
=\frac{1}{T-\lfloor T/2\rfloor+1}
\sum_{t=\lfloor T/2\rfloor}^{T} e_t,
\label{eq:steady_state_auc}
\end{equation}
For even $T$ this gives $\lfloor T/2\rfloor+1$ terms; the boundary effect is negligible for the trajectory lengths used. The quantity is the last-50\% time-average RMSPBE after discarding the initial transient. Tables report the sample mean and sample standard deviation of $\operatorname{AUC}_{\mathrm{ss}}$ and of the final RMSPBE across 100 runs. Curves show mean $\pm$ sample standard deviation, not confidence intervals. Divergent runs are not clipped.

\begin{table}[!t]
\centering
\caption{Hyperparameter search spaces used for tuning. The objective is the average RMSPBE over the last 20\% of the tuning trajectory on eight disjoint tuning seeds. Final results use 100 separate evaluation seeds.}
\label{tab:hyperparams}
\resizebox{\textwidth}{!}{%
\begin{tabular}{lll}
\toprule
Algorithm & Tuned parameters & Search values \\
\midrule
TD, GTD2-MP & $\alpha$ & $\{0.001,0.003,0.005,0.01,0.03,0.05,0.1\}$ \\
GTD2, TDC & $\alpha,\beta_w$ & $\alpha\in\{0.0003,0.001,0.003,0.005,0.01,0.03,0.05,0.1\}$; $\beta_w=\alpha z$, $z\in\{0.05,0.1,0.25,0.5,1,2,4,8\}$, $\beta_w\leq0.1$ \\
TDRC & $\alpha,\alpha_w,\eta$ & $\alpha,\alpha_w\in\{0.0003,0.001,0.003,0.005,0.01,0.03,0.05,0.1\}$; $\eta\in\{0.01,0.03,0.1,0.3,1.0\}$ \\
BA-TDC & $\alpha,\alpha_w$ & $\alpha,\alpha_w\in\{0.0003,0.001,0.003,0.005,0.01,0.03,0.05,0.1\}$ \\
BA-TDRC & $\alpha,\alpha_w,\beta$ & $\alpha,\alpha_w\in\{0.0003,0.001,0.003,0.005,0.01,0.03,0.05,0.1\}$; $\beta\in\{0.1,0.3,0.7,1,2,3,5,10\}$ \\
\bottomrule
\end{tabular}}
\end{table}

\subsection{Main Comparison}

\begin{figure}[!t]
\centering
\includegraphics[width=0.95\textwidth]{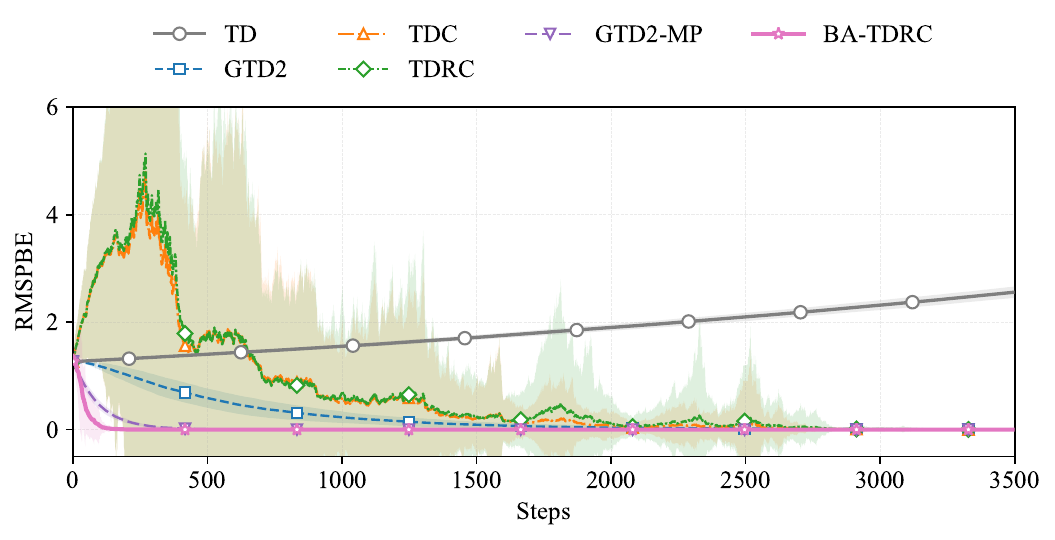}
\caption{Main comparison on the two-state counterexample. Curves show mean RMSPBE over 100 runs and shaded regions show one sample standard deviation.}
\label{fig:two_state_curve}
\end{figure}

\begin{figure}[!t]
\centering
\includegraphics[width=0.95\textwidth]{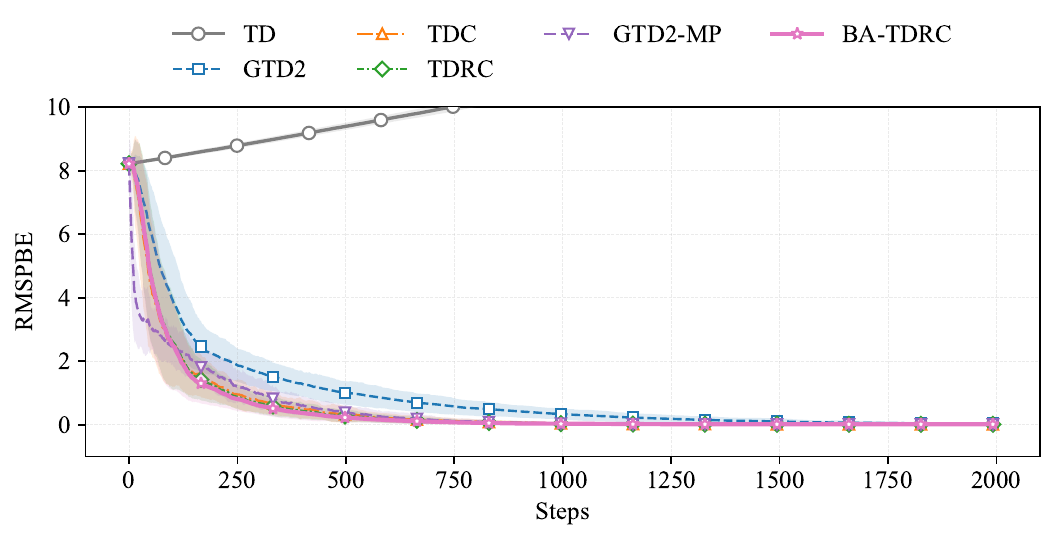}
\caption{Main comparison on Baird's counterexample. Curves show mean RMSPBE over 100 runs and shaded regions show one sample standard deviation.}
\label{fig:baird_curve}
\end{figure}

\begin{figure}[!t]
\centering
\includegraphics[width=0.95\textwidth]{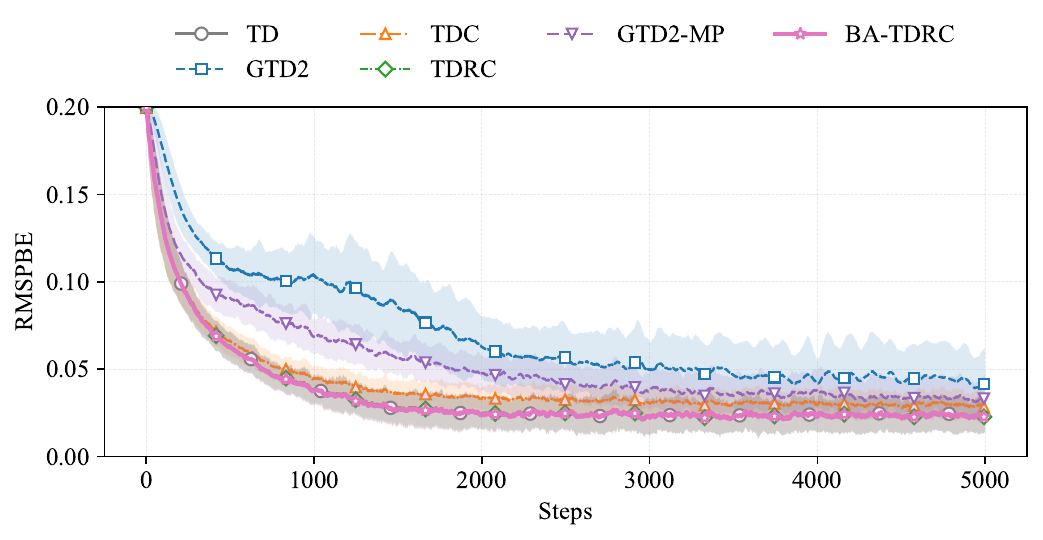}
\caption{Main comparison on Random Walk. Curves show mean RMSPBE over 100 runs and shaded regions show one sample standard deviation.}
\label{fig:random_walk_curve}
\end{figure}

\begin{figure}[!t]
\centering
\includegraphics[width=0.95\textwidth]{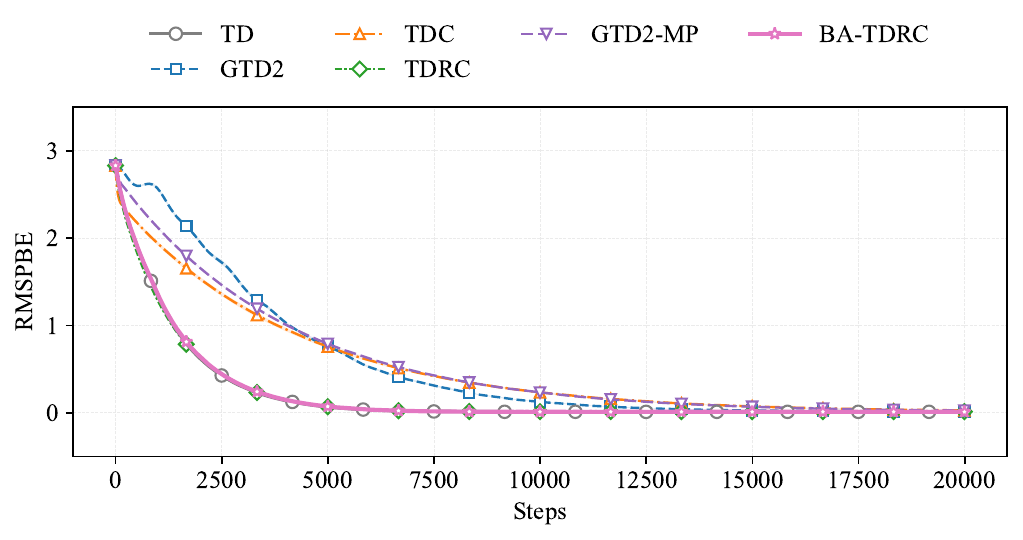}
\caption{Main comparison on Boyan Chain. Curves show mean RMSPBE over 100 runs and shaded regions show one sample standard deviation.}
\label{fig:boyan_chain_curve}
\end{figure}

Figures~\ref{fig:two_state_curve}--\ref{fig:boyan_chain_curve} compare BA-TDRC with the main baselines. BA-TDRC reaches near-zero RMSPBE on the two-state counterexample, is competitive with TDC and TDRC on Baird's counterexample, matches the strongest TD-style methods on Random Walk, and obtains the best result on Boyan Chain. These results support BA-TDRC as a regularized behavior-aware correction, but they do not imply uniform dominance over every baseline in every environment.

\begin{table}[!t]
\centering
\caption{Steady-state AUC error (last-50\% time-average RMSPBE), mean $\pm$ sample standard deviation over 100 runs. Lower is better.}
\label{tab:auc}
\resizebox{\textwidth}{!}{%
\begin{tabular}{lcccc}
\toprule
Algorithm & Two-state & Baird & Random Walk & Boyan Chain \\
\midrule
TD & $2.718\pm0.109$ & $12.203\pm0.233$ & $\mathbf{0.0236\pm0.0024}$ & $0.0115\pm0.0008$ \\
GTD2 & $3.14\times10^{-3}\pm3.04\times10^{-3}$ & $0.1244\pm0.0817$ & $0.0475\pm0.0053$ & $0.0397\pm0.0021$ \\
TDC & $6.33\times10^{-3}\pm3.41\times10^{-2}$ & $0.0153\pm0.0092$ & $0.0306\pm0.0023$ & $0.0916\pm0.0036$ \\
TDRC & $1.05\times10^{-2}\pm5.55\times10^{-2}$ & $0.0162\pm0.0093$ & $0.0237\pm0.0024$ & $0.0133\pm0.0007$ \\
GTD2-MP & $2.89\times10^{-13}\pm3.46\times10^{-13}$ & $\mathbf{0.0128\pm0.0053}$ & $0.0363\pm0.0022$ & $0.0879\pm0.0010$ \\
BA-TDRC & $\mathbf{9.00\times10^{-43}\pm8.87\times10^{-42}}$ & $0.0151\pm0.0086$\textsuperscript{$\dagger$} & $\mathbf{0.0236\pm0.0024}$ & $\mathbf{0.0114\pm0.0008}$ \\
\bottomrule
\end{tabular}}
\vspace{0.4ex}
\begin{minipage}{0.98\textwidth}
\footnotesize \textsuperscript{$\dagger$} The Hurwitz condition in Assumption~\ref{ass:hurwitz} is not satisfied under this hyperparameter setting (see Table~\ref{tab:numerical}); this result is an empirical stress test, not covered by Theorem~\ref{thm:convergence}.
\end{minipage}
\end{table}

\begin{table}[!t]
\centering
\caption{Final RMSPBE at the last step, mean $\pm$ sample standard deviation over 100 runs. Lower is better.}
\label{tab:final}
\resizebox{\textwidth}{!}{%
\begin{tabular}{lcccc}
\toprule
Algorithm & Two-state & Baird & Random Walk & Boyan Chain \\
\midrule
TD & $3.452\pm0.174$ & $13.838\pm0.349$ & $\mathbf{0.0236\pm0.0090}$ & $0.0109\pm0.0041$ \\
GTD2 & $1.59\times10^{-4}\pm2.52\times10^{-4}$ & $0.0308\pm0.0325$ & $0.0419\pm0.0204$ & $0.0157\pm0.0071$ \\
TDC & $1.09\times10^{-4}\pm4.74\times10^{-4}$ & $0.0104\pm0.0077$ & $0.0290\pm0.0074$ & $0.0270\pm0.0047$ \\
TDRC & $3.72\times10^{-5}\pm1.24\times10^{-4}$ & $0.0137\pm0.0089$ & $0.0236\pm0.0090$ & $0.0127\pm0.0046$ \\
GTD2-MP & $4.41\times10^{-23}\pm6.96\times10^{-23}$ & $\mathbf{0.0082\pm0.0008}$ & $0.0339\pm0.0132$ & $0.0254\pm0.0039$ \\
BA-TDRC & $\mathbf{1.91\times10^{-86}\pm1.02\times10^{-85}}$ & $0.0136\pm0.0086$\textsuperscript{$\dagger$} & $\mathbf{0.0236\pm0.0090}$ & $\mathbf{0.0108\pm0.0041}$ \\
\bottomrule
\end{tabular}}
\vspace{0.4ex}
\begin{minipage}{0.98\textwidth}
\footnotesize \textsuperscript{$\dagger$} The Hurwitz condition in Assumption~\ref{ass:hurwitz} is not satisfied under this hyperparameter setting (see Table~\ref{tab:numerical}); this result is an empirical stress test, not covered by Theorem~\ref{thm:convergence}.
\end{minipage}
\end{table}

Tables~\ref{tab:auc} and~\ref{tab:final} show the main quantitative pattern without the ablation-only BA-TDC variant. BA-TDRC is strongest on the two-state and Boyan Chain tasks, slightly improves over TDC and TDRC on Baird in steady-state AUC, and matches the best Random Walk result. The Baird entries marked by $\dagger$ are not covered by the Hurwitz condition used in Theorem~\ref{thm:convergence}; they are retained to show the empirical stress-test behavior under the same tuning protocol. GTD2-MP remains best on Baird, so BA-TDRC should be viewed as improving the TDC/TDRC correction family rather than replacing all gradient-TD alternatives.

On Random Walk, off-policy bias is mild because the behavior policy is uniform ($\mu(\cdot|s)=0.5$) and the target policy is only slightly biased ($\pi(\text{right}|s)=0.6$), giving importance ratios $\rho\in\{0.8,1.2\}$ close to one. Under such mildly off-policy conditions, the tuned BA-TDRC hyperparameters yield a large effective regularization on the auxiliary recursion, so the auxiliary variable $w_t$ stays close to zero and the primary update reduces to off-policy semi-gradient TD. The tuned TDRC selects a similarly small auxiliary effect. This explains why BA-TDRC, TDRC, and TD coincide on Random Walk: the corrected and uncorrected updates are numerically indistinguishable in this regime. This also agrees with the mean-operator analysis in Table~\ref{tab:numerical}: TDRC has a slightly smaller deterministic factor there, but the difference is too small to produce a visible advantage in the stochastic finite-horizon RMSPBE curves.

\subsection{Modular Ablation}

\begin{figure}[!t]
\centering
\includegraphics[width=0.95\textwidth]{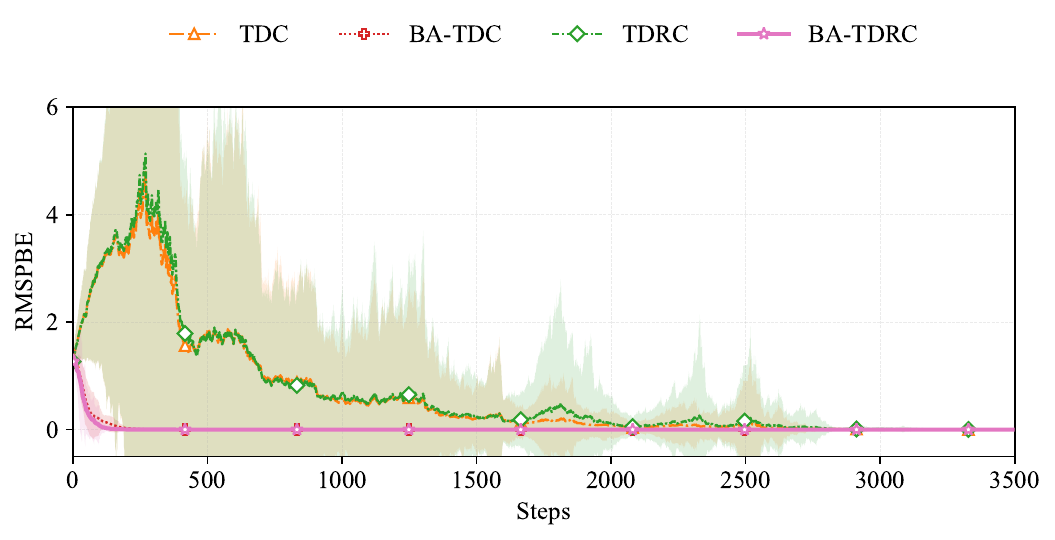}
\caption{Ablation on the two-state counterexample. The comparison isolates TDC $\rightarrow$ BA-TDC and TDRC $\rightarrow$ BA-TDRC.}
\label{fig:ablation_two_state}
\end{figure}

\begin{figure}[!t]
\centering
\includegraphics[width=0.95\textwidth]{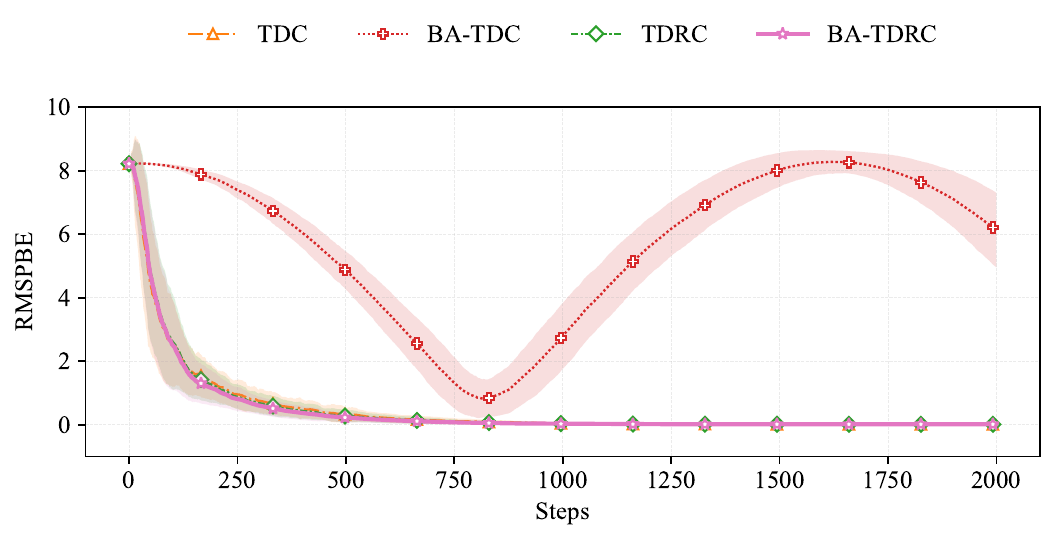}
\caption{Ablation on Baird's counterexample. The comparison isolates TDC $\rightarrow$ BA-TDC and TDRC $\rightarrow$ BA-TDRC.}
\label{fig:ablation_baird}
\end{figure}

\begin{figure}[!t]
\centering
\includegraphics[width=0.95\textwidth]{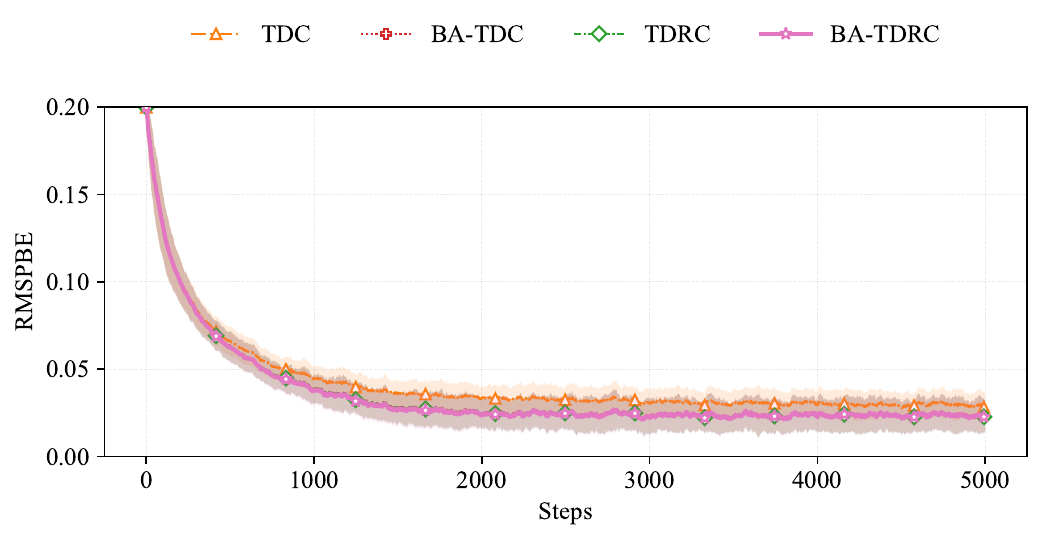}
\caption{Ablation on Random Walk. The comparison isolates TDC $\rightarrow$ BA-TDC and TDRC $\rightarrow$ BA-TDRC.}
\label{fig:ablation_random_walk}
\end{figure}

\begin{figure}[!t]
\centering
\includegraphics[width=0.95\textwidth]{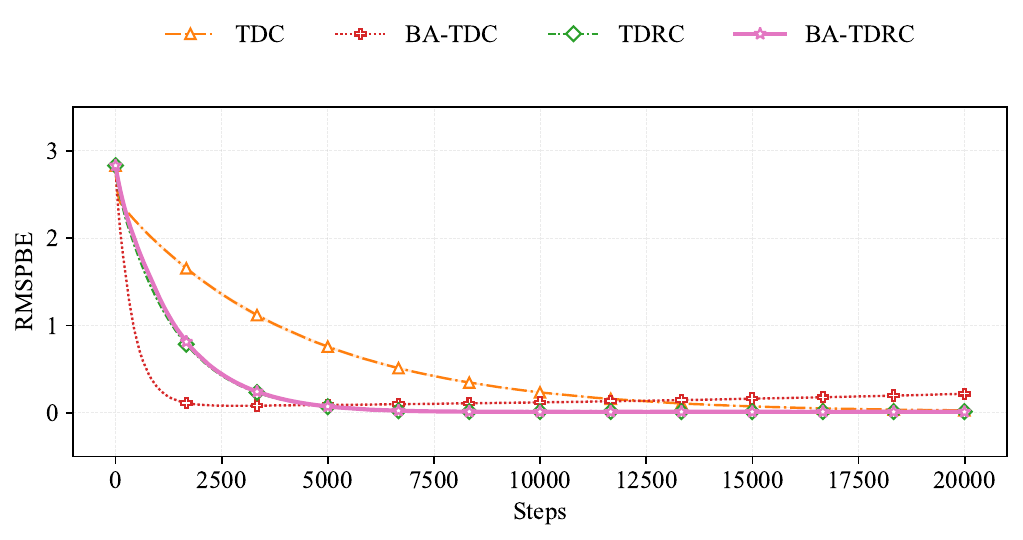}
\caption{Ablation on Boyan Chain. The comparison isolates TDC $\rightarrow$ BA-TDC and TDRC $\rightarrow$ BA-TDRC.}
\label{fig:ablation_boyan_chain}
\end{figure}

\begin{table}[!t]
\centering
\caption{Ablation AUC results for the covariance-to-behavior replacement and regularization. Lower is better.}
\label{tab:ablation_auc}
\resizebox{\textwidth}{!}{%
\begin{tabular}{lcccc}
\toprule
Algorithm & Two-state & Baird & Random Walk & Boyan Chain \\
\midrule
TDC & $6.33\times10^{-3}\pm3.41\times10^{-2}$ & $0.0153\pm0.0092$ & $0.0306\pm0.0023$ & $0.0916\pm0.0036$ \\
BA-TDC & $7.19\times10^{-23}\pm4.43\times10^{-22}$ & $6.809\pm0.383$ & $0.0240\pm0.0024$ & $0.1637\pm0.0014$ \\
TDRC & $1.05\times10^{-2}\pm5.55\times10^{-2}$ & $0.0162\pm0.0093$ & $0.0237\pm0.0024$ & $0.0133\pm0.0007$ \\
BA-TDRC & $\mathbf{9.00\times10^{-43}\pm8.87\times10^{-42}}$ & $\mathbf{0.0151\pm0.0086}$ & $\mathbf{0.0236\pm0.0024}$ & $\mathbf{0.0114\pm0.0008}$ \\
\bottomrule
\end{tabular}}
\end{table}

The ablation shows that the behavior-aware replacement and regularization have different roles. Replacing $C$ by $A_\mu$ alone is highly effective on the two-state counterexample and helpful on Random Walk, but it fails on Baird and Boyan Chain. This is expected from the structure of the unregularized auxiliary equation: $A_\mu$ contains temporal feature transport and is not a symmetric covariance matrix, so in aliased or strongly coupled feature spaces its inverse can amplify noise and transient residuals rather than damping them. Baird's counterexample is especially sensitive because the feature representation is over-parameterized relative to the state dynamics and the importance ratios are extreme. In this setting, the nonsymmetric behavior-aware auxiliary direction is repeatedly excited by rare high-ratio solid transitions, producing the nonmonotone growth visible in the ablation curve. Boyan Chain has milder ratios but strongly correlated features. Adding $\beta I$ shifts the auxiliary matrix away from such poorly conditioned behavior-aware inversions, preserves the two-state gain, and removes the severe BA-TDC failures.

\subsection{Step-Size Robustness}

\begin{figure}[!t]
\centering
\includegraphics[width=0.98\textwidth]{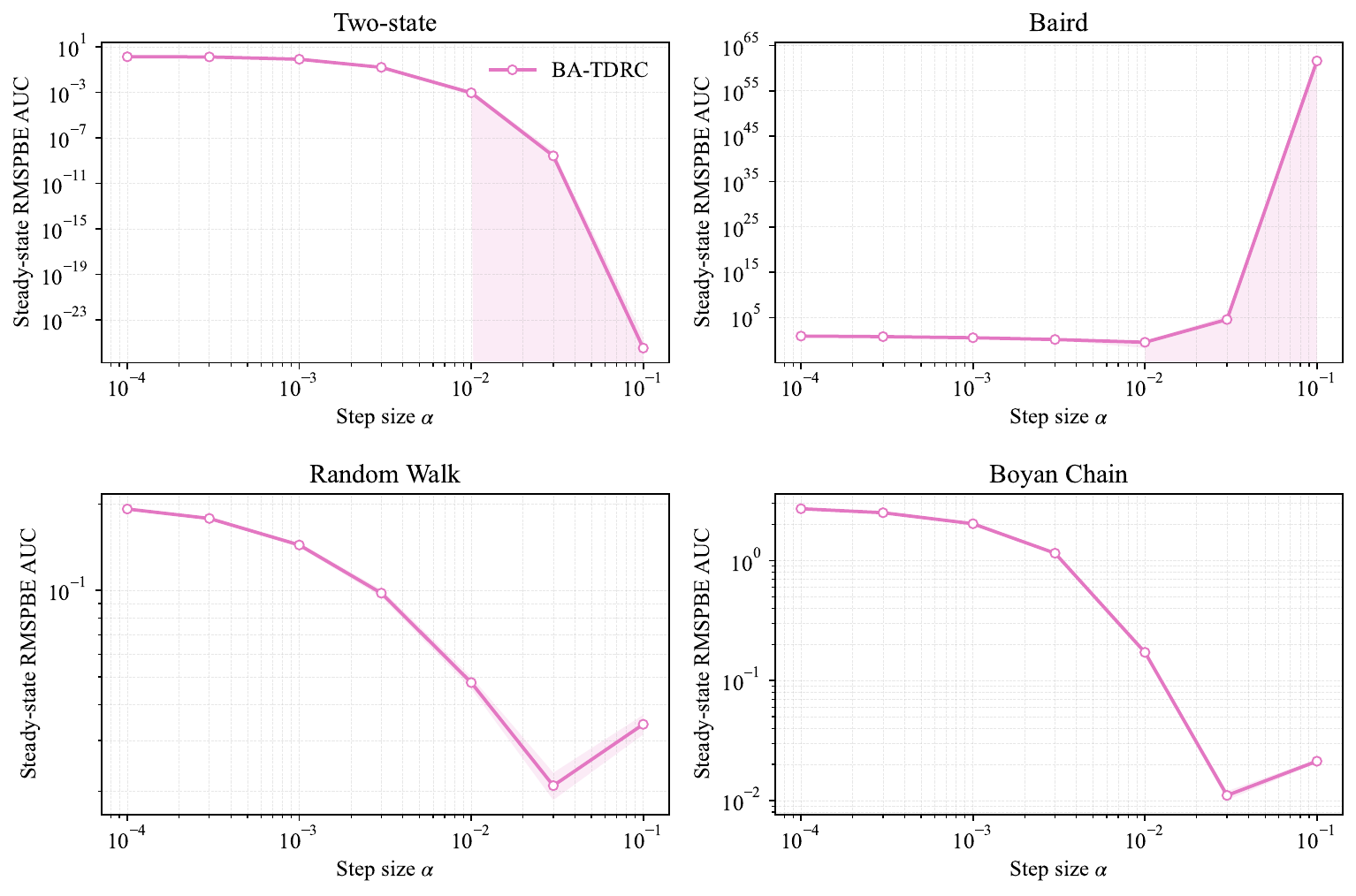}
\caption{BA-TDRC step-size robustness across the four off-policy benchmarks. Each panel reports steady-state RMSPBE AUC against the primary step size $\alpha$ over 100 seeds, with the regularizer fixed to $\beta=1.0$ and the auxiliary-to-primary gain ratio fixed to $\lambda=1$; these fixed values differ from the per-environment tuning used in the main comparison.}
\label{fig:robustness_grid}
\end{figure}

Figure~\ref{fig:robustness_grid} reports BA-TDRC robustness alone, rather than mixing robustness curves from unrelated baselines. The four panels show how sensitive BA-TDRC is to the primary step size in each environment. The two-state case has a wide low-error region at larger step sizes, Random Walk and Boyan Chain have smoother low-error regions, and Baird is more sensitive, consistent with its role as a difficult off-policy counterexample. The vertical axes are logarithmic and span very different ranges because BA-TDRC reaches extremely small RMSPBE on the two-state problem, whereas high step sizes on Baird can amplify the already difficult off-policy transient by many orders of magnitude. In these robustness plots, the regularizer is fixed to $\beta=1.0$ and the auxiliary-to-primary gain ratio is fixed to $\lambda=1$ across all environments, so that only the primary step size $\alpha$ varies along the horizontal axis. This is a controlled robustness sweep and is therefore not the same hyperparameter setting as the per-environment BA-TDRC in the main comparison, where $\beta$ and $\lambda$ are tuned per environment.

\section{Numerical Mean-Operator Analysis}\label{sec:numerical}

For each benchmark, we construct the exact finite-state matrices $A_\pi$, $A_\mu$, $C$, and $D_\pi$ and instantiate the two auxiliary matrices in the theory: $M_C=C+\eta I$ for TDRC and $M_A=A_\mu+\beta I$ for BA-TDRC. The values of $\eta$, $\beta$, and the auxiliary gain ratios are the tuned values used in the main experiments. Table~\ref{tab:numerical} reports three checks that connect the analysis to the experiments: (i) whether $M_A-D_\pi$ is nonsingular, so the TD fixed point is preserved by Proposition~\ref{prop:td_fixed_point}; (ii) the Hurwitz margins for the TDRC and BA-TDRC mean systems, where a positive value verifies the corresponding mean-system stability condition; and (iii) whether the behavior-aware speed advantage condition in Assumption~\ref{ass:speed}, namely $q_{M_A,\lambda_A}<q_{M_C,\lambda_C}$, holds.

\begin{table}[!t]
\centering
\caption{Exact mean-operator verification of the theory. The column $\sigma_{\min}(M_A-D_\pi)$ checks fixed-point preservation; positive Hurwitz margins indicate stable continuous-time mean systems; lower best $q$ is faster and verifies Assumption~\ref{ass:speed}.}
\label{tab:numerical}
\resizebox{\textwidth}{!}{%
\begin{tabular}{lccccccc}
\toprule
Environment & $\sigma_{\min}(M_A-D_\pi)$ & TDRC margin & BA margin & Best $q_C$ & Best $q_A$ & $q_A<q_C$ & Interpretation \\
\midrule
Two-state & $1.525$ & $0.0265$ & $0.4875$ & $0.9749$ & $0.5745$ & yes & large BA speed advantage \\
Baird & $0.224$ & $4.96\times10^{-18}$ & $-3.92\times10^{-4}$ & $1.0000$ & $1.0000$ & no & BA near-critical \\
Random Walk & $9.827$ & $0.0059$ & $0.0318$ & $0.9941$ & $0.9970$ & no & TDRC mean factor smaller \\
Boyan Chain & $9.794$ & $0.0145$ & $0.0243$ & $0.9855$ & $0.9857$ & no & speed condition not verified \\
\bottomrule
\end{tabular}}
\end{table}

The numerical analysis forms the bridge between the theory and the experiments. First, $\sigma_{\min}(M_A-D_\pi)>0$ in all four benchmarks, so the condition used to prove preservation of the same TD fixed point is verified. Second, the Hurwitz condition required for the almost-sure convergence theorem is verified for the tuned two-state, Random Walk, and Boyan Chain BA-TDRC mean systems, but not for Baird under the selected constant-gain configuration. The Baird BA-TDRC margin is only slightly negative ($-3.92\times10^{-4}$), while the corresponding TDRC margin is essentially zero and positive only at numerical precision ($4.96\times10^{-18}$). Thus both regularized correction systems are close to the stability boundary on Baird. With the tuned finite step size and regularization, the stochastic recursion remains empirically well controlled over the evaluated horizon, but this behavior is outside the theorem's coverage and should be read as a stress test of the update. Third, the speed assumption $q_{M_A,\lambda_A}<q_{M_C,\lambda_C}$ holds clearly on the two-state counterexample, matching the large empirical gain. The other three benchmarks do not satisfy the deterministic mean-speed condition under the tuned RMSPBE settings, so Proposition~\ref{prop:faster} is not invoked for those cases.

This gap between the mean-rate condition and the empirical results is informative. Assumption~\ref{ass:speed} compares asymptotic deterministic linear factors of the exact mean recursion; it does not account for finite-horizon bias, Markovian sampling variance, or the way regularization changes the magnitude and variability of the auxiliary vector along stochastic trajectories. On Random Walk, the off-policy mismatch is mild and the tuned regularization drives $w_t$ close to zero, making BA-TDRC behave like off-policy semi-gradient TD. On Boyan Chain, the speed factors of TDRC and BA-TDRC are nearly identical ($0.9855$ versus $0.9857$), so the deterministic asymptotic comparison is too fine to explain the observed AUC difference; finite-sample damping and auxiliary-variance effects can dominate over the small mean-factor disadvantage. The present experiments support this interpretation qualitatively through the ablation and robustness patterns, but they do not provide a separate variance decomposition. A complete theory connecting behavior-aware regularization to stochastic finite-sample RMSPBE remains open.

\section{Conclusion}

We proposed behavior-aware auxiliary corrections for off-policy temporal-difference prediction. BA-TDC isolates the replacement of the covariance matrix $C$ by the behavior Bellman matrix $A_\mu$, while BA-TDRC combines the same replacement with TDRC-style regularization. We gave the mean-system formulation, established convergence under a Hurwitz condition verified from exact finite-state matrices for three of the four tuned benchmark settings, derived a conditional mean-rate comparison, and evaluated the modular increments on four standard off-policy benchmarks. Baird serves as an empirical stress test outside the coverage of the convergence theorem. The RMSPBE results show that behavior-aware geometry alone can be very effective on the two-state counterexample but unreliable on harder tasks, whereas the regularized BA-TDRC variant is competitive with or better than TDC/TDRC across the tested prediction problems. The method should therefore be viewed as a modular correction-geometry change whose benefit depends on the behavior-induced mean operator and its interaction with regularization, rather than as a uniformly dominant replacement for TDRC on every off-policy prediction task. Extending the idea to neural-network critics is a natural next step, but it requires controlling the additional error from learned, time-varying feature maps and from online estimates of behavior-aware feature-transition operators.

\section*{Data and Code Availability}
The experimental code is publicly available at \url{https://github.com/GameAI-NJUPT/BA-TDRC}. Generated result tables and figures are not included in the repository and can be reproduced by running the scripts with the documented protocols.

\bibliographystyle{elsarticle-num}
\bibliography{references}

\end{document}